\documentclass[a4paper,UKenglish,cleveref, autoref, thm-restate]{oasics-v2021}
\pdfoutput=1 

\usepackage{booktabs}
\usepackage{amsthm}
\usepackage{algpseudocode}
\usepackage{algorithm}
\usepackage{graphicx}

\nolinenumbers

\title{Early Failure Prediction from Near-Anomaly Detection: A Proactive Approach} 

\author{Léa Billet}{Schaeffler, Toulouse, France \and LAAS-CNRS, University of Toulouse, INSA, Toulouse, France}{lea.billet@mail.schaeffler.com}{https://orcid.org/0009-0005-1360-9944}{}

\author{Louise Travé-Massuyès}{LAAS-CNRS, University of Toulouse, CNRS, Toulouse, France}{louise@laas.fr}{https://orcid.org/0000-0002-5322-8418}{}

\author{Elodie Chanthery}{LAAS-CNRS, University of Toulouse, INSA, Toulouse, France}{elodie.chanthery@laas.fr}{https://orcid.org/0000-0003-0015-5566}{}

\author{Alexandre Gaffet}{Schaeffler, Toulouse, France}{alexandre.gaffet@mail.schaeffler.com}{https://orcid.org/0000-0002-5347-8062}{}

\authorrunning{L. Billet, L. Travé-Massuyès, E. Chanthery, and A. Gaffet} 

\Copyright{L. Billet, L. Travé-Massuyès, E. Chanthery, and A. Gaffet} 



\relatedversion{} 

\acknowledgements{The authors would like to thank Christophe Merle, Head of the Manufacturing Intelligent 5 team at
Schaeffler and ANITI Industrial Coordinator, for sharing his expertise in robotized production lines. Our work has benefited from the AI Interdisciplinary Institute ANITI. ANITI is funded by the France 2030 program under the Grant agreements n°ANR-19-P3IA-0004 and n°ANR-23-IACL-0002.}

\EventEditors{Ingo Pill, Marina Zanella, and Gregory Provan}
\EventNoEds{3}
\EventLongTitle{37th International Conference on Principles of Diagnosis and Resilient Systems (DX 2026)}
\EventShortTitle{DX 2026}
\EventAcronym{DX}
\EventYear{2026}
\EventDate{September 1--3, 2026}
\EventLocation{Cork, Ireland}
\EventLogo{}
\SeriesVolume{148}
\ArticleNo{5}

\begin{document}

\maketitle

\begin{abstract}
Anomaly detection methods often have uncertain behavior with respect to samples near the distribution boundary, limiting their ability to anticipate future anomalies. This work introduces the concept of near-anomalies that, while not yet anomalous, lie close to the boundary and are likely to transition into anomalies in the near future. To address this, we propose an unsupervised method, named Christoffel-based ANomaly Anticipation for eaRly dIscovery (CANARI), which leverages the strong theoretical foundations of the Christoffel function to detect near-anomalies. The method is validated on industrial in-circuit testing data from printed circuit boards, with synthetically generated near-anomaly samples due to the lack of real-world data labeling. Experimental results show that CANARI outperforms the compared baselines that generally use a dual-threshold mechanism (one for anomalies and one for near-anomalies). It therefore provides a proactive solution for anticipating anomalies before they occur, offering a promising approach for resilience, predictive maintenance, and quality control.
\end{abstract}

\section{Introduction}

Outliers, or anomalies, are typically defined as samples that are rare or out-of-distribution compared to inliers~\cite{ruff2021unifying}. However, in many real-world applications, potentially problematic samples may not yet qualify as true anomalies at the time of observation. Instead, they may lie near the boundary of the normal distribution, remaining technically within its limits. We refer to such samples as near-anomalies. While not classical outliers, they are of critical interest because their detection enables the early anticipation of future anomalies, thereby enhancing system resilience. In this work, we focus on identifying these near-anomaly samples—those that, though still part of the distribution, are precariously close to its edge. The method is applied to data from industrial manufacturing plant product test stations. 

Near-anomaly detection has been explored in image anomaly detection, where only a few pixels deviate from the normal distribution~\cite{strater2024generalad}. The generation of near-anomaly images has also been proposed for method evaluation~\cite{mirzaei2022fake}, notably using diffusion-based models. However, near-anomaly detection for high-dimensional tabular data remains relatively unexplored, despite its strong potential for predicting later problems.

Tabular anomaly detection methods typically score samples within a distribution and apply a threshold to classify them as inliers or outliers~\cite{samariya2023comprehensive}. They are usually trained in a semi-supervised manner, using only inlier data. In principle, these methods can be adapted to detect near-anomalies by introducing a second threshold closer to the anomaly boundary. This near-anomaly threshold may be set in a supervised manner, for example, using the dataset’s known contamination rate or defining it as a percentage of the primary anomaly threshold. However, in an unsupervised setting, such a hyperparameter cannot be automatically determined for all datasets.

In this paper, we propose a method, named Christoffel-based ANomaly Anticipation for eaRly dIscovery (CANARI), based on the Christoffel Function (CF), where a second threshold to detect near-anomaly samples is derived by leveraging the theoretical properties of the CF. While we use the CF-based method CLOE (CLOE)~\cite{billet2026cloe} for its scalability with respect to the feature space dimension, any other CF-based method, like DyCF \cite{ducharlet2024leveraging}, could equally apply even if it is more limited with respect to the dimension. 
CLOE relies on an autoencoder trained with a CF regularized loss to construct a well-structured, low-dimensional representation of the data, which can then be advantageously used for anomaly detection. We extend the CLOE framework to identify samples that, while not yet outliers, lie close enough to the decision boundary to be classified as near-anomalies—potential precursors to future anomalies.

Experiments are conducted on printed circuit boards (PCBs) using in-circuit testing (ICT) tabular data acquired on industrial manufacturing plant product test stations. In manufacturing plants, products undergo extensive testing during production to ensure quality, with results typically recorded as high-dimensional tabular data. Quality engineers define knowledge-based thresholds for each test, and products whose values remain near these thresholds may still pass inspection and proceed through the manufacturing process. Nevertheless, these borderline products can later exhibit issues in subsequent production steps or, more seriously, after delivery to the customer. Because thresholds are defined before all production features and distribution characteristics are fully known, such products are not necessarily labeled as anomalies at the time of testing. A possible approach would be to lower the test station thresholds in order to classify more products as anomalies. However, the appropriate reduction for each threshold is unknown and cannot be reliably determined in advance. Moreover, dependencies may exist between different testing station thresholds, which makes manual adjustment even more difficult. As a result, a more flexible method is needed to identify products that are not yet anomalous but are sufficiently close to abnormal conditions to warrant attention. Detecting them as near-anomalies would therefore provide an opportunity to raise early warnings of potential future quality problems.

The paper is organized as follows: background about the CF is presented in Section~\ref{sec:background}. Then the method CANARI is described in Section~\ref{sec:method} introducing the concept of near-anomaly and its detection. In Section~\ref{sec:experiments}, the set of conducted experiments and the associated results are presented with a description of the industrial use case. Finally, the conclusion and future work are given in Section~\ref{sec:conclusion}.

\section{Background - The Christoffel function for anomaly detection} \label{sec:background}
This section defines the CF and establishes the background for our method~\cite{lasserre2019empirical}, \cite{lasserre2022christoffel}.

\subsection{The Christoffel Function}
Let $\mu$ be a finite Borel measure \footnote{A \textit{Borel measure} is a way to assign a "size" (like length, area, or probability) to subsets of a topological space (here $\Omega$), where the subsets we can measure are all those that can be built from open sets using countable unions and complements. On the real line, the usual \textit{length} is a Borel measure which tells how long any interval is, and by extension, how long more complex sets (like unions of intervals) are.} supported on $\Omega \subset \mathbb{R}^d$, a compact set with a non-empty interior. Let $\mathbf{x}=(x_1, x_2, \cdots, x_d) \in \mathbb{R}^d$ and let $\alpha=(\alpha_i)_{i=1...d}\in\mathbb{N}^d$ be the vector of degrees associated to each variable for the monomial $\mathbf{x}^{\alpha}:=x_1^{\alpha_1}x_2^{\alpha_2}...x_d^{\alpha_d}$ of total degree $deg(\alpha)=\sum_{i=1}^d\alpha_i$. Let $\mathbb{N}_n^d=\{\mathbb{N}^d; deg(\alpha) \leq n\}$ and $\mathbf{v}_n(\mathbf{x}):= (\mathbf{x}^{\alpha})_{\alpha \in \mathbb{N}_n^d}$ be the vector of all monomials of degree less or equal to $n$ graded in the lexicographic order\footnote{Lexicographic order means that monomials are first sorted by degree; then, within each degree, they are ordered lexicographically by the variables, with $X_1=a$, $X_2=b$, etc.}. The size of the vector $\mathbf{v}_n(\mathbf{x})$ is equal to $s_d(n)=\binom{d+n}{n}$.

Let $\mathbf{M}_n(\mu)$ be the moment matrix of $\mu$. This matrix encodes information that quantifies the location, the shape, and the scale of $\Omega$. $\mathbf{M}_n(\mu)$ is a real symmetric, positive definite, and non-singular matrix for all $n$, and is given by 
\begin{equation}
     \mathbf{M}_n(\mu) = \int_{\mathbb{R}^d}\mathbf{v}_n(\mathbf{\mathbf{x}}) \mathbf{v}_n(\mathbf{x})^Td\mu(\mathbf{\mathbf{x}})
\end{equation}

Let us introduce the Christoffel-Darboux kernel $K_n^{\mu}$ associated with $\mu$. Given any orthonormal basis $(p_i)^{s_d(n)}_{i=1}$ of $\mathbb{R}_n[\mathbf{x}]$, orthonormal with respect to the inner product induced by $\mathbf{M}_n(\mu)$, $K_n^{\mu}$ is defined as:
\begin{equation} \label{eq:defCDKernel}
    (\mathbf{x},\mathbf{y}) \mapsto K_n^{\mu}(\mathbf{x},\mathbf{y}):= \sum_{i=1}^{s_d(n)} p_i(\mathbf{x})p_i(\mathbf{y}).
\end{equation}
This kernel can also be expressed in terms of the moment matrix:
\begin{equation}
    K_n^{\mu}(\mathbf{x},\mathbf{y}):=\mathbf{v}_n(\mathbf{\mathbf{x}})^T \mathbf{M}_n(\mu)^{-1}\mathbf{v}_n(\mathbf{\mathbf{y}}).
\end{equation}

\begin{definition}[The Christoffel Function]
The Christoffel Function (CF) of degree $n \in \mathbb{N}$ associated with the measure $\mu$, denoted by $\Lambda_n^{\mu}$, is given for any $\mathbf{x} \in \mathbb{R}^d$ as:
\begin{equation} \label{eq:defCF}
\Lambda_n^{\mu}(\mathbf{x})= K_n^{\mu}(\mathbf{x},\mathbf{x})^{-1},
\end{equation}
which can be shown to be the solution of this minimization problem:
\begin{equation}
    \label{eq:cf_int_p}
    \Lambda^\mu_n(\mathbf{x}) = \underset{P \in \mathbb{R}_n[\mathbf{x}]}{min}\left\{ \int_{\Omega} P(\mathbf{z})^2 ~ d\mu(\mathbf{z}), \quad P(\mathbf{x}) = 1 \right\}\,,\quad\forall \mathbf{x}\in\mathbb{R}^d\,.
\end{equation}
\end{definition}

The intuition behind the Christoffel function can be understood from the minimization problem \eqref{eq:cf_int_p}. The solution polynomial $P$ of \eqref{eq:cf_int_p} is constrained to satisfy $P(\mathbf{x}) = 1$ at a given point $\mathbf{x}$ while minimizing the $L^2$ norm $\int_\Omega P^2 \, d\mu$ over the support $\Omega$. If $\mathbf{x} \notin \Omega$ (outside the support), $P$ can be nearly zero on $\Omega$ while still satisfying $P(\mathbf{x}) = 1$, resulting in a small minimal value. Conversely, if $\mathbf{x} \in \Omega$ (inside the support), the constraint $P(\mathbf{x}) = 1$ forces $P$ to have significant magnitude on $\Omega$, preventing it from being close to zero and thus yielding a large minimal value.

One of the main and salient features of the CF is its ability to encode the support $\Omega$. The geometric shape of $\Omega$ can indeed be captured by a specific level set $\Omega_\gamma:=\{\mathbf{x}: \Lambda_n^{\mu}(\mathbf{x})^{-1}\leq \gamma\}$, defined for some $\gamma \in \mathbb{R}_+$, even for low degrees $n$ \cite{ducharlet2024leveraging}.

\subsection{The empirical Christoffel Function}
In practical applications, the theoretical measure $\mu$ is unknown. Instead, one has access to a cloud of data points $\mathbb{X}=\{\mathbf{x}_i,i=1,\dots,N\}$ sampled from $\mu$ and associated with a discrete measure $\mu_N := \frac{1}{N} \sum_{\mathbf{x} \in \mathcal{X}} \delta_{\mathbf{x}}$ supported on $\mathbb{X}$ (where $\delta_\mathbf{x}$ stands for the Dirac measure at $\mathbf{x}$).

The empirical version of the moment matrix writes as 
\begin{equation}
    \mathbf{M}_n(\mu_N) = \frac{1}{N}\sum_{\mathbf{x} \in\mathbb{X}_e}\mathbf{v}_n(\mathbf{x}) \mathbf{v}_n(\mathbf{x})^T.
\end{equation}

To guarantee the invertibility of the matrix $\mathbf{M}_n(\mu_N)$, the size of $\mathbb{X}$ must be greater than $s_d(n)$ according to \cite{lasserre2022christoffel} (Corollary 6.3.5). Under the condition $|\mathbb{X}|=N >s_d(n)$, the inverse of the empirical CF (eCF) is defined as:
\begin{equation}\label{eq:inv_CF}
    \Lambda_n^{\mu_N}(\mathbf{x})^{-1} := \mathbf{v}_n(\mathbf{x})^T \mathbf{M}_n(\mu_N)^{-1}\mathbf{v}_n(\mathbf{x}), \mathbf{x} \in \mathbb{X}.
\end{equation}

It has been proven that the eCF converges to the CF as $N$ tends to infinity~\cite{lasserre2022christoffel} (Theorem 6.2.3). This convergence ensures that the properties of the CF also hold (under some conditions) for the eCF.

Consequently, one can expect the inverse eCF $\Lambda_n^{\mu_N}(\mathbf{x})^{-1}$ to capture the support $\mathbb{X}$ thanks to a specific level set. Defining \emph{inliers} as data points in the support and \emph{outliers} or \emph{anomalies} as data points out of the support, $\Lambda_n^{\mu_N}(\mathbf{x})^{-1}$ hence comes as a good scoring function for outlier detection. 

\section{Near-anomaly detection} \label{sec:method}
The goal of this work is to automatically detect potentially problematic samples that, while not yet true anomalies, may lie near the boundary of the normal distribution although remaining technically within its limits. To achieve this, we propose to use the dichotomy property of the growth of the CF with the degree $n$ presented in the following section.

\subsection{The Christoffel function theoretical growth}
In addition to being able to capture the support of a discrete measure defined by a cloud of data points, the CF has been proven to have drastically different theoretical growths with the degree $n$ for data points in the support of the measure, considered as inliers, and for data points out of the support of the measure, considered as outliers.

Two theorems from~\cite{lasserre2022christoffel} characterize the growth of the CF as the degree $n$ increases. The growth with $n$ for data points in the support (inliers) is shown to be polynomial in Theorem~\ref{th:lemma2} whereas it is shown to be exponential for data points out of the support (outliers) in Theorem~\ref{th:lemma1}. Both cases are illustrated in Figure \ref{fig:growth} (for an outlier sample with the red graph and for an inlier sample with the green graph).

\begin{theorem} \label{th:lemma2}(Lemma 4.3.2, p.51 \cite{lasserre2022christoffel}) Let $\mu$ be a positive Borel measure supported on the compact set $\Omega$, the closure of a bounded domain $U$ with a nice boundary. Assume $\mu|_U$ has a density with respect to the Lebesgue measure $\lambda$ restricted to $U$, bounded from below by $c>0$. Let $\mathbf{x} \in \Omega$ and $\delta>0$ such that $ \operatorname{dist}(\mathbf{x}, \partial U) \geq \delta$. Then, for any $n \in \mathbb{N}, n>0$, and considering the equality \eqref{eq:defCF}, we have:
\begin{equation}\label{eq:growth_inliers}
    \Lambda_n^{\mu}(\mathbf{x})^{-1} \leq s_d(n) \frac{2\lambda(\Omega)}{c \delta^d \omega_d}(1+d)^3
\end{equation}
where $\omega_d$ depends only on $d$ and not on $n$.
\end{theorem}

In equation \ref{eq:growth_inliers} of Theorem \ref{th:lemma2}, for fixed \(d\), $s_d(n)$ is a polynomial in \(n\) of degree \(d\):
    \[
    s_d(n)=\binom{n+d}{n} = \frac{(n+d)(n+d-1)\cdots(n+1)}{d!} = \frac{n^d}{d!} + \text{(lower-order terms in } n).
    \]
So, as \(n \to \infty\), it behaves like:
    \begin{equation}\label{eq:binomial}
        \binom{n+d}{n} \sim \frac{n^d}{d!}.
    \end{equation} 
    Theorem \ref{th:lemma2} hence shows that $\Lambda_n^{\mu}(\mathbf{x})^{-1}$ has polynomial growth with $n$ inside the support.
 \begin{theorem} \label{th:lemma1}(Lemma 4.3.1, p.50 \cite{lasserre2022christoffel}) Let $\mu$ be a positive Borel measure supported on the compact set $\Omega \subset \mathbb{R}^d$. Let $\mathbf{x} \notin \Omega$ and $\delta>0$ such that $ \operatorname{dist}(\mathbf{x}, \Omega) \geq \delta$. Then, for any $n \in \mathbb{N}, n>0$, and considering the equality \eqref{eq:defCF}, we have:
\begin{equation}\label{eq:growth_outliers}
    \frac{\Lambda_n^{\mu}(\mathbf{x})^{-1}}{s_d(n)} \geq 2^{\frac{\delta n}{\delta + diam(\Omega)}-3}n^{-d}\left ( \frac{d}{e} \right)^d exp\left (- \frac{d^2}{n} \right)
\end{equation}
\end{theorem}
Theorem \ref{th:lemma1} clearly shows that $\Lambda_n^{\mu}(\mathbf{x})^{-1}$ has exponential growth with $n$ outside the support.

\begin{figure}[t]
    \centering
    \includegraphics[width=0.9\linewidth]{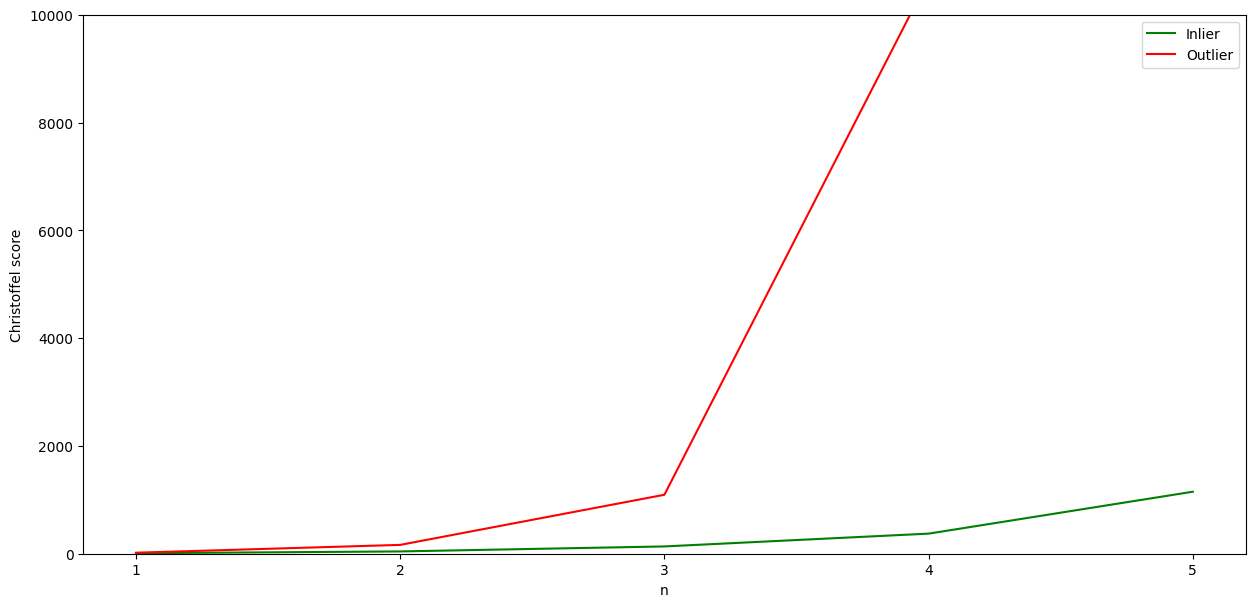}
    \caption{The figure represents the growth of an outlier in the red graph and of an inlier in the green graph as $n$ grows from 1 to 5. The red graph increases much faster than the green graph, showing the dichotomy property stated by Theorems~\ref{th:lemma2} and \ref{th:lemma1}.}
    \label{fig:growth}
\end{figure}

This theoretical growth with respect to the parameter $n$ has already been leveraged by DyCG~\cite{ducharlet2024leveraging}, which proposes a parameter-free outlier detection approach for data streams. DyCG evaluates the difference between two eCFs computed for distinct values of the parameter $n$, normalized by the difference between the corresponding values of $n$. For a fixed $x \in \mathbb{X}$, this score computes the slope of the function depending on $n$.

\subsection{The Christoffel function growth ratio}
To leverage the dichotomy property stated by Theorems \ref{th:lemma2} and \ref{th:lemma1}, and given that $\Lambda^{\mu}_{n}(\mathbf{x})^{-1}$ is strictly positive, we propose to define the following \emph{CF growth ratio}, based on two different values of $n$, namely $n_1$ and $n_2$, $n_1 < n_2$ to compare eCF scores.

\begin{definition}[CF growth ratio]
Given $\mathbf{x} \in \mathbb{X}$, the \emph{CF growth ratio} is defined for $n_1, n_2 \in \mathbb{N}^*, n_1 < n_2$ as follows:
    \begin{equation}\label{eq:ratio1}
    \frac{\Lambda^{\mu}_{n_2}(\mathbf{x})^{-1}}{\Lambda^{\mu}_{n_1}(\mathbf{x})^{-1}}
\end{equation}
\end{definition}

A ratio was preferred to a normalized difference, such as the one used in DyCF, for two main reasons: 
\begin{enumerate}
\item The ratio eliminates the effect of the absolute scale of the scores and avoids normalization, enabling a direct comparison of the degree's impact on the metric without bias related to the magnitude of the raw values,
\item The CF growth ratio comes with two interesting asymptotic properties.
\end{enumerate}
The relative magnitude of the eCFs obtained for two different values of $n$ can be captured through their ratio. Figure~\ref{fig:growthRatio} illustrates the eCF score for $n_1=1$ along the horizontal axis and the eCF score for $n_2=4$ along the vertical axis. This figure highlights high values for outliers, denoted by red crosses, while inliers, denoted by green crosses, have lower values. The growth ratio characterizes the relative change of the CF as $n$ increases, whereas DyCG~\cite{ducharlet2024leveraging} characterizes the absolute change. Figure~\ref{fig:growthRatioScore} displays the ratio between two eCFs for $n_1=1$ and $n_2=2$ for inliers, in green, and outliers, in red. Consistently to the dichotomy property and Figure~\ref{fig:growthRatio}, outlier samples have the highest score.

\begin{figure}[t]
    \centering
    \includegraphics[width=0.9\linewidth]{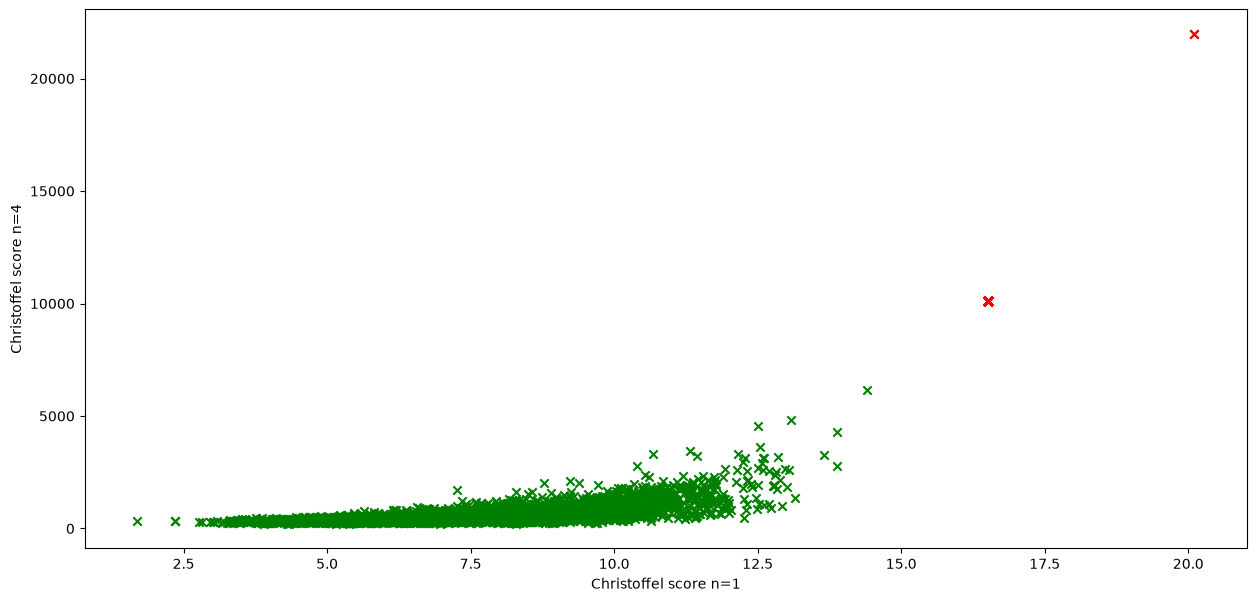}
    \caption{The figure illustrates the eCF obtained for $n_1=1$ alongside that obtained for $n_2=4$. Inliers are in green and outliers are in red. Outlier and inlier scores are clearly distinguishable.}
    \label{fig:growthRatio}
\end{figure}

\begin{figure}[t]
    \centering
    \includegraphics[width=0.9\linewidth]{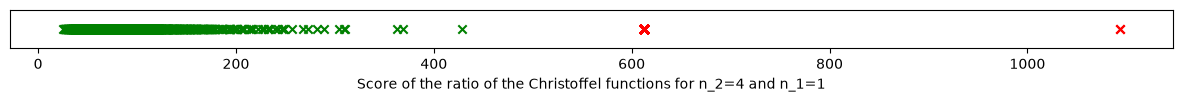}
    \caption{The figure presents the values of the ratio between the two eCFs computed for $n_1=1$ and $n_2=4$. Inliers are represented in green, while the outliers are represented in red. A clear separation is observed, outliers are associated with much higher scores than inliers.}
    \label{fig:growthRatioScore}
\end{figure}

\begin{corollary}\label{cor:inliers}
Given $\mathbf{x} \in \mathbb{X}$, and $n_1, n_2 \in \mathbb{N}^{*}, n_1<n_2$, the CF growth ratio of $\mathbf{x}$ behaves like:
\begin{equation}\label{eq:ratio1.1}
    \frac{\Lambda^{\mu}_{n_2}(\mathbf{x})^{-1}}{\Lambda^{\mu}_{n_1}(\mathbf{x})^{-1}} \sim \left(\frac{n_2}{n_1}\right)^p,
\end{equation}
where $p \in \mathbb{N}^*$ is a constant.
\end{corollary}

\begin{proof}
Given $\mathbf{x} \in \mathbb{X}$ and $n \in \mathbb{N}^{*}$, from Theorem \ref{th:lemma2}, there exist $C_d \in \mathbb{R}^{*+}$ and $p \in \mathbb{N}^*$ such that, $\Lambda^{\mu}_n(\mathbf{x})^{-1}$ behaves like:
\begin{equation}
    \Lambda^{\mu}_n(\mathbf{x})^{-1} \sim C_dn^p
\end{equation}
Taking $n_1, n_2 \in \mathbb{N}^*, n_1 < n_2$, we have $\Lambda^{\mu}_{n_1}(\mathbf{x})^{-1} \sim C_d(n_1)^{p}$ and $\Lambda^{\mu}_{n_2}(\mathbf{x})^{-1} \sim C_d(n_2)^{p}$, which proves Corollary \ref{cor:inliers}.
\end{proof}

\begin{remark}
    Given $\mathbf{x} \in \mathbb{X}$, and $n_1, n_2 \in \mathbb{N}^{*}, n_1 \to \infty$ and $n_2 \to \infty$, the CF growth ratio of $\mathbf{x}$ behaves asymptotically like:
\begin{equation}\label{eq:ratio1.2}
    \frac{\Lambda^{\mu}_{n_2}(\mathbf{x})^{-1}}{\Lambda^{\mu}_{n_1}(\mathbf{x})^{-1}} \sim \left(\frac{n_2}{n_1}\right)^d,
\end{equation}
\end{remark}
\begin{corollary}\label{cor:outliers}
Given $\mathbf{x} \notin \mathbb{X}$, and $n_1, n_2 \in \mathbb{N}^{*}, n_1<n_2$, the CF growth ratio of $\mathbf{x}$ behaves like:
\begin{equation}\label{eq:ratio1.3}
    \frac{\Lambda^{\mu}_{n_2}(\mathbf{x})^{-1}}{\Lambda^{\mu}_{n_1}(\mathbf{x})^{-1}} \sim e^{(n_2-n_1)b},
\end{equation}
where $b \in \mathbf{R}^{*+}$ is a constant. 
\end{corollary}

\begin{proof}
Given $\mathbf{x} \notin \mathbb{X}$ and $n\in \mathbb{N}^{*}$, from Theorem \ref{th:lemma1}, there exist $a, b \in \mathbf{R}^{+*}$ such that $\Lambda^{\mu}_n(\mathbf{x})^{-1}$ behaves like: 
\begin{equation}
    \Lambda^{\mu}_n(\mathbf{x})^{-1} \sim ae^{nb}
\end{equation}
Taking $n_1, n_2 \in \mathbb{N}^{*}, n_1 < n_2$, we have $\Lambda^{\mu}_{n_1}(\mathbf{x})^{-1} \sim ae^{n_1b}$ and $\Lambda^{\mu}_{n_2}(\mathbf{x})^{-1} \sim ae^{n_2b}$, which proves Corollary \ref{cor:outliers}.
\end{proof}

\subsection{Near-anomaly definition}
Given that the eCF inherits the properties of the CF, this section leverages the growth properties of the CF given by Theorems \ref{th:lemma2} and \ref{th:lemma1} and the two Corollaries \ref{cor:inliers} and \ref{cor:outliers} to define near-anomalies in a practical setting.

\begin{definition}[eCF growth ratio]
Given $\mathbf{x} \in \mathbb{X}$, the \emph{eCF growth ratio} $\mathcal{G}_{n_2/n_1}($ is defined for $n_1, n_2 \in \mathbb{N}^{*}, n_1 < n_2$ as follows:
    \begin{equation}\label{eq:ratio2}
    \mathcal{G}_{n_2/n_1}(\mathbf{x}) =\frac{\Lambda^{\mu_N}_{n_2}(\mathbf{x})^{-1}}{\Lambda^{\mu_N}_{n_1}(\mathbf{x})^{-1}}
\end{equation}
\end{definition}
The following definition defines \emph{near-anomalies} thanks to the eCF growth ratio.
 \begin{definition}[Near-anomaly definition] Given the eCF growth ratio for $n_1, n_2 \in \mathbb{N}^{*}, n_1 < n_2$, a data point $\mathbf{x}\in \mathbb{X}$ is defined as a near-anomaly if $\mathcal{G}_{n_2/n_1}(\mathbf{x})>\tau_{na}$, where $\tau_{na} \in \mathbb{R}$ is a given threshold.
 \end{definition}
This definition proposes to define near-anomalies as data points that are inliers but whose eCF shows a growth significantly superior to core inliers. Interestingly, in the following section, we derive a method to set automatically the near-anomaly threshold $\tau_{na}$.

 \subsection{Automatic determination of the near-anomaly threshold}
 The results provided in this section allow to derive a method to set automatically the near-anomaly threshold.

From Corollaries \ref{cor:inliers} and \ref{cor:outliers}, the eCF growth ratio $\mathcal{G}_{n_2/n_1}(\mathbf{x})$ is expected to remain relatively low for inliers ($\mathbf{x} \in \mathbb{X}$, where $\mathbb{X}$ corresponds to the support of $\mu_N$) while being much higher for outliers/anomalies ($\mathbf{x} \notin \mathbb{X}$). Near-anomalies are inliers but they lie in a range in which the eCF growth ratio is significantly superior to that of core inliers. Let us define as $\mathbb{X}_{c} \subset \mathbb{X}$ the set of core inliers and by $\mathbb{X}_{na} \subset \mathbb{X}$ the set of near-anomalies that we want to identify. For this purpose, there is a theoretical result that can be useful. The following lemma indeed gives the theoretical mean of the eCF on the support of the distribution.

\begin{lemma}[Theoretical mean value of eCF in the support] \label{th:lemma3}
    The theoretical mean value of the eCF over the support of $\mu_N$ is $s_d(n)$.
\end{lemma}

\begin{proof}
\begin{equation} \label{eq:meanCF}
    \begin{split}
        \int \Lambda^{\mu_N}_n(\mathbf{x})^{-1} d\mu_N(\mathbf{x}) & = \int K_n^{\mu_N}(\mathbf{x},\mathbf{x}) d\mu_N(\mathbf{x}) ~~(\text{from Equation } \eqref{eq:defCF})\\
     & = \sum_{i=1}^{s_d(n)} \int p_i(\mathbf{x})^2 d\mu_N(\mathbf{x}) ~~(\text{from Equation }\eqref{eq:defCDKernel})\\
     & = \sum_{i=1}^{s_d(n)} 1 ~~(\text{as }(p_i)^{s_d(n)}_{i=1} \text{ is an orthonormal basis of } \mathbb{R}_n[\mathbf{x}]) \\ 
     & = s_d(n)
    \end{split}
\end{equation}
\end{proof}

To leverage this result, we propose to use the Chebyshev inequality \cite{pukelsheim1994three} to identify the inliers low and high eCF growth ratio ranges. According to this inequality, for random distributions, at least $100(\frac{k^2-1}{k^2})\%$ of the samples are in the range given by the following interval:
\begin{equation}\label{sigma_rule}
 [\bar{x} - k \times \sigma ; \bar{x} + k \times \sigma],
\end{equation}
where $\bar{x}$ and $\sigma$ are the mean and the standard deviation of the distribution, respectively. We further assume that the interval upper bound $(\bar{x} + k \times \sigma)$ can be used as a threshold to identify high values and hence near-anomalies, which would amount to no more than $100\left(\frac{1}{k^2}\right)\%$ of the samples.

Let us apply Chebyshev inequality to our problem, which translates in the following definition.

\begin{definition}[Core inliers and near-anomaly samples]\label{def:core_near}
Let us define $\mathbb{E}\left[\mathcal{G}_{n_2/n_1}\right]$ and $\sigma\left(\mathcal{G}_{n_2/n_1}\right)$ as the mean and the standard deviation of the eCF growth ratio $\mathcal{G}_{n_2/n_1}$, core inliers and near-anomaly samples are characterized as follows:
\begin{align}
&\forall \mathbf{x} \in \mathbb{X}_c \Leftrightarrow \mathcal{G}_{n_2/n_1}(\mathbf{x}) \in \left]0 ~ ; \mathbb{E}\left[\mathcal{G}_{n_2/n_1}\right] + k \times \sigma\left(\mathcal{G}_{n_2/n_1}\right)\right], \\
&\forall \mathbf{x} \in \mathbb{X}_{na} \Leftrightarrow \mathcal{G}_{n_2/n_1}(\mathbf{x}) > \mathbb{E}\left[\mathcal{G}_{n_2/n_1}\right] + k \times \sigma\left(\mathcal{G}_{n_2/n_1}\right).
\end{align}
\end{definition}
 
From Definition \ref{def:core_near}, it comes that the threshold $\tau_{na}$ that we want to determine is:
\begin{equation}\label{eq:tau}
\tau_{na}=\mathbb{E}\left[\mathcal{G}_{n_2/n_1}\right] + k \times \sigma\left(\mathcal{G}_{n_2/n_1}\right).
\end{equation}

In consequence, our problem comes down to determining $\mathbb{E}\left[\mathcal{G}_{n_2/n_1}\right]$ and $\sigma\left(\mathcal{G}_{n_2/n_1}\right)$. The main result below provides a way to compute $\tau_{na}$ thanks to Lemma \ref{th:lemma3}, which provides the theoretical mean value of eCF in the support, and some empirical computations.
\begin{theorem} [Practical near-anomaly threshold] \label{eq:tnatilde}
    Consider $n_1, n_2 \in \mathbb{N}^{*}$, $n_1<n_2$, $k \in \mathbb{N}^{*}$, and $m, \Gamma \in \mathbb{R}^{+*}$ , then a practical near-anomaly threshold $\tilde{\tau}_{na}$ is given by:
    \begin{align}
        &\tilde{\tau}_{na} = \frac{s_d(n_2)}{m}+k\times \Gamma \\
    & \text{where } s_d(n_2)=\binom{n_2+d}{n_2}, \label{eq:xbarre}\\
    &m = \operatorname{min}\left\{\left(\Lambda^{\mu_N}_{n_1}(\mathbf{x})\right)^{-1}, \mathbf{x} \in \mathbb{X}\right\}, \label{eq:minValid}\\
    &\Gamma = \sqrt{\operatorname{var}\left\{{\mathcal{G}_{n_2/n_1}(\mathbf{x}), \mathbf{x} \in \mathbb{X}}\right\}}. \label{eq:sigmaValid}
    \end{align}
\end{theorem}
\begin{proof}
From \eqref{eq:tau}, we have $\tau_{na}=\mathbb{E}\left[\mathcal{G}_{n_2/n_1}\right] + k \times \sigma\left(\mathcal{G}_{n_2/n_1}\right)$ and we have:
\begin{equation*}
\mathbb{E}\left[\mathcal{G}_{n_2/n_1}\right]= \int \mathcal{G}_{n_2/n_1}(\mathbf{x})d\mu_N(\mathbf{x})= \int \frac{ \Lambda^{\mu_N}_{n_2}(\mathbf{x})^{-1}}{ \Lambda^{\mu_N}_{n_1}(\mathbf{x})^{-1}} d\mu_N(\mathbf{x})
\end{equation*}
As $\Lambda^{\mu_N}_{n_1}(\mathbf{x})^{-1}$ is lower bounded on $\mu_N$, there exists a constant $m \in \mathbb{R}^{+*}$ such that $\Lambda^{\mu_N}_{n_1}(\mathbf{x})^{-1} \geq m$ \cite{lasserre2022christoffel}. So,
\begin{equation*} \label{eq:quotienIntegral} 
    \int \frac{ \Lambda^{\mu_N}_{n_2}(\mathbf{x})^{-1}}{ \Lambda^{\mu_N}_{n_1}(\mathbf{x})^{-1}} d\mu_N(\mathbf{x}) \leq \int \frac{ \Lambda^{\mu_N}_{n_2}(\mathbf{x})^{-1}}{ m} d\mu_N(\mathbf{x})=\frac{1}{m} \int \Lambda^{\mu_N}_{n_2}(\mathbf{x})^{-1} d\mu_N(\mathbf{x})
\end{equation*}
and from Lemma \ref{th:lemma3}:
\begin{equation*}
 \int \frac{ \Lambda^{\mu_N}_{n_2}(\mathbf{x})^{-1}}{ \Lambda^{\mu_N}_{n_1}(\mathbf{x})^{-1}} d\mu_N(\mathbf{x}) \leq \frac{s_d(n_2)}{m} \Leftrightarrow \tau_{na} \leq \frac{s_d(n_2)}{m} + k \times \sigma(\mathcal{G}_{n_2/n_1}). 
\end{equation*}
Because there is no theoretical result providing $m$ and $\sigma(\mathcal{G}_{n_2/n_1})$, we propose to compute them empirically from the data set $\mathbb{X}$ at hand\footnote{To avoid division by extremely small values, a $99\%$ percentile can be estimated from the validation set.}. 
Finally, 
\begin{equation*}
    \tau_{na} \leq \frac{s_d(n_2)}{\operatorname{min}\left\{\left(\Lambda^{\mu_N}_{n_1}(\mathbf{x})\right)^{-1}, \mathbf{x} \in \mathbb{X}\right\}} + k \times \sqrt{\operatorname{var}\left\{{\mathcal{G}_{n_2/n_1}(\mathbf{x}), \mathbf{x} \in \mathbb{X}}\right\}}= \tilde{\tau}_{na}.
\end{equation*}
\end{proof}

$\tilde{\tau}_{na}$ is more restrictive than $\tau_{na}$ but, as will be seen in the experiments, it provides a good practical threshold to identify near-anomalies. Algorithm~\ref{alg:threshold} proposes the pseudo-code to compute the $\tilde{\tau}_{na}$ threshold.

\begin{algorithm}[t]

\caption{Computation of $\tilde{\tau}_{na}$} \label{alg:threshold}
\begin{algorithmic} 
\State \textbf{Input:} $n_2 > n_1 > 0$, 
\State \textbf{Input:} $k>0$,
\State \textbf{Input:} Trained $(\Lambda^{\mu_N}_{n_1})^{-1}$ and $(\Lambda^{\mu_N}_{n_2})^{-1}$,
\State \textbf{Input:} Validation dataset $X_v=\{x_i\}_{i \in (1, |X_v|)}$
\State \textbf{Output:} The threshold $\tilde{\tau}_{na}$
\State $d \gets$ dimension of $x_1$
\State $\mathcal{G}_{n_2/n_1} \gets \left\{\frac{\Lambda^{\mu_N}_{n_2}(\mathbf{x})^{-1}}{\Lambda^{\mu_N}_{n_1}(\mathbf{x})^{-1}}, x \in X_v\right\}$ \Comment{Equation~\ref{eq:ratio2}}
\State $\bar{x} \gets s_d(n_2)=\binom{n_2+d}{n_2}$ \Comment{Equation~\ref{eq:xbarre}}
\State $min \gets min\left\{\Lambda^{\mu_N}_{n_1}(\mathbf{x})^{-1}, x \in X_v\right\}$ \Comment{Equation~\ref{eq:minValid}}
\State $\sigma \gets \sqrt{\operatorname{Var}(\mathcal{G}_{n_2/n_1})}$ \Comment{Equation~\ref{eq:sigmaValid}}
\State $\tilde{\tau}_{na} \gets \frac{\bar{x}}{min_{n_1}} + k \times \sigma$ \Comment{Theorem~\ref{eq:tnatilde}}
\State \Return $\tilde{\tau}_{na}$
\end{algorithmic}
\end{algorithm}

\subsection{Scaling up the Christoffel function}

\begin{figure}[t]
    \centering
    \includegraphics[width=1\linewidth]{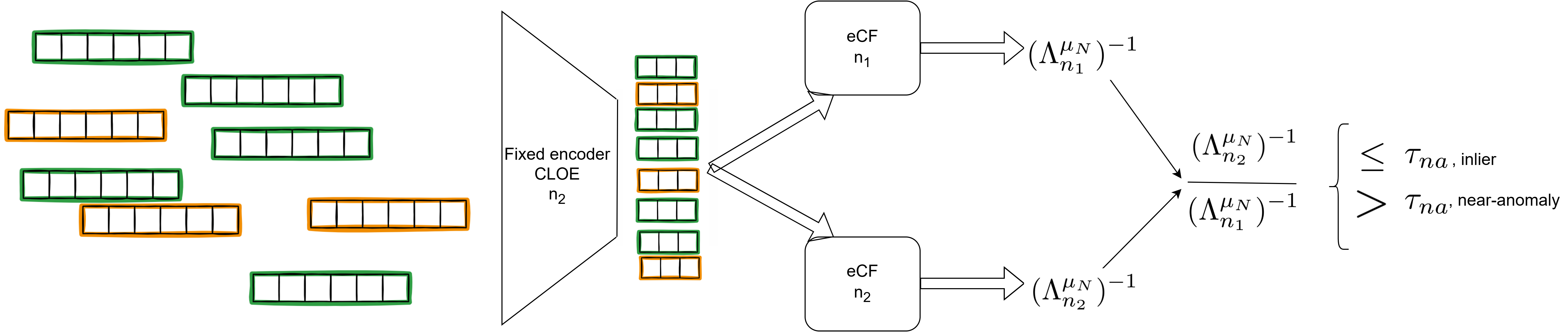}
    \caption{With normal data as input, the trained encoder from CLOE delivers low dimension encoded data in the latent space. Both near-anomaly samples, in orange, and core inlier samples, in green, are present. Then two eCF are trained, one with parameter $n_1$ and the other with $n_2, n_1<n_2$. Finally, the ratio of the two eCF scores is compared to the near-anomaly threshold, $\tau_{na}$ (in practice $\tilde{\tau}_{na})$, to decide if the data is core inlier or near-anomaly.}
    \label{fig:fullMethod}
\end{figure}

Even for moderate degrees $n$, the inverse of the Christoffel function is very expensive to compute if not impossible for feature dimensions $d>8$ because of the necessity to invert the moment matrix $\mathbf{M}_n(\mu_N)$ of size $s_d(n) \times s_d(n)$ in equation \eqref{eq:inv_CF}. 

To use the proposed method in high-dimensional settings, we propose to use the Christoffel-based Loss for Autoencoder (CLOE)~\cite{billet2026cloe}. CLOE consists of a method that applies the eCF to high-dimensional data for semi-supervised anomaly detection. It allows to reduce the dimension of high-dimensional datasets in a latent space specifically regularized with the CF score in the loss function of the autoencoder (AE).

CLOE includes a pretraining stage and a joint training stage. The pretraining stage is performed using only the Mean Square Error (MSE) between the input and the output of the AE as the loss function. Then, the joint training stage combines the MSE with the CF-based loss to structure the latent space. For this latter step, the only hyperparameter required is the degree $n$ for the CF computation, which is set to $n=n_{CLOE}$. The trained encoder reduces the data dimension to a lower dimension in the latent space (8 dimensions are proposed in \cite{billet2026cloe}). The latent space data is then used to train the eCF. The threshold for outlier detection is chosen as the maximum value of the eCF. 

In this work, CLOE is used as follows. CLOE is trained with normal data, the parameter $n_{CLOE}$ being set to $n_2$, which provides a low dimension latent space characterized by the set of moments up to the order $n_2$. After this dimension reduction phase, the obtained low dimension latent space data is used to compute the eCF for the two degrees $n=n_1$ and $n=n_2$, $n_1<n_2$, i.e., to compute $\Lambda^{\mu_N}_{n_1}(\mathbf{x})^{-1}$ and $\Lambda^{\mu_N}_{n_2}(\mathbf{x})^{-1}$, and the CANARI method can be applied to identify near-anomalies. This is illustrated in Figure~\ref{fig:fullMethod}.

\section{Experiments and results} \label{sec:experiments}

\subsection{Datasets} \label{sec:dataset}
CANARI is applied to real production data acquired in industrial plants. In these plants, PCBs are produced in an almost fully automated way. First, the PCBs are assembled in the front-end line. Components are placed and welded during this step. At the end of this line, ICT is performed. This test consists of injecting current into the PCB at many points on the circuit (between 10 and 2000, depending on the size of the PCB) and measuring the voltage and the current intensity to check the conformance of the PCB. Test engineers define limits for each test, and the PCB continues the manufacturing process if all its ICT results are within the limits. Otherwise, the PCB is removed from the production line. The second line is the back-end line where the PCB is encapsulated in an appropriate box. Finally, the PCB can be delivered to the customer.

The ICT results constitute tabular data with dimensions ranging from 10 to 1200, depending on the number of test points for each PCB. 8 different PCBs produced in 2023 were considered for this study, resulting in eight datasets D1 to D8. The ICT results may be PASS or FAIL. Only products that pass the ICT constitute the datasets. The dimensions of these products vary between 18 and 1139. The details of the dimension and the size for each dataset are indicated in Table~\ref{tab:dimDataset}. Each dataset is divided into 3 different sets: the training set, composed of 1150 samples, the validation set, composed of 4600 samples, and the testing set composed of the rest of the samples. 

\begin{table}[t] 
    \centering 
    \begin{tabular}{c|c|c|c|c|c|c|c|c}
         Dataset& D1 & D2 & D3 & D4 & D5 & D6 & D7 & D8 \\
        \midrule
         Dimension &26&821&80&691&703&714&1139&18\\
         Nb samples &57183&56334&59947&25827&10091&8421&26729&41151\\
    \end{tabular}
    \caption{Dimension and size of each dataset} \label{tab:dimDataset} 
\end{table}

\begin{figure}[t]
    \centering
    \includegraphics[width=1\linewidth]{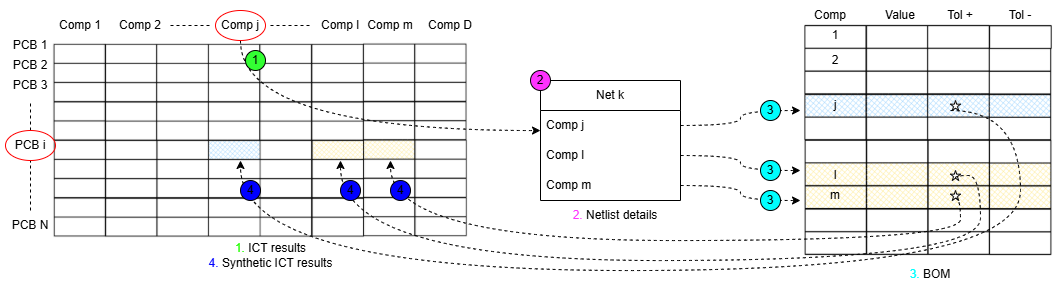}
    \caption{1. A PCB product (PCB i) and one of its component (Comp j) are randomly chosen (red circle). 2. Netlist is checked to find the net (k) to which comp j belongs. All the components of the net k (j, l and m) are modified. 3. The value of these components and their upper and lower tolerances Tol+ and Tol-, respectively, are obtained from the Bill Of Materials (BOM), from which the upper and lower limits, $l^+$ and $l^-$ respectively, are computed. 4. $l^+$ or $l^-$ is randomly chosen. The corresponding ICT results of the concerned components are modified to appear below or above their limit.}
    \label{fig:syntDataset}
\end{figure}

To evaluate the method, synthetic anomalies were injected into the datasets, as illustrated in Figure~\ref{fig:syntDataset}. Starting from real datasets and known ICT threshold values and upper and lower tolerances Tol+ and Tol-, respectively, from the Bill Of Materials (BOM), $10\%$ of the samples in each test dataset were modified. First, a component is randomly chosen. Then the upper limit $l^+$ or lower limit $l^-$ is randomly chosen. The ICT result for this component is modified to be close to the limit $l$. The order of magnitude $o$ of the ICT result is computed with :
\begin{equation} \label{eq:orderMagnitude}
    o = 10^{\lfloor log(l^+ - l^-) \rfloor}.
\end{equation} The final value of the modified component is $l \pm o \times \tau_{synth}$, where $\tau_{synth}$ is the disturbed
magnitude of the near anomaly samples. Finally, knowing the design of the PCB, all the components connected by the first net in the netlist, where the component belongs, are modified according to the same upper or lower decision. With this process, synthetic datasets are created taking into account the connection of the different components on the PCB and the ICT.

\subsection{Baselines and Metrics}

CLOE with double threshold is considered as the first baseline. The original threshold is considered to compute the new threshold to detect near-anomalies. 5\%, 20\% and 50\% of the original threshold are used as the double threshold to detect near-anomaly samples.

Then, CANARI with different thresholds is used as the second baseline. The new thresholds are the ratio of the theoretical means of the two eCFs, $\tau_{mean}$, and the empirical percentile $99^{th}$ of the ratio, $\tau_{upper}$.

F1-score and Matthews Correlation Coefficient (MCC) \cite{matthews1975comparison} are the two metrics used to evaluate CANARI. F1-score, given in Definition~\ref{eq:f1score}, is computed with the recall and precision metrics. Its lower value is 0 and its best value is 1. 

\begin{definition}[F1-score] \label{eq:f1score}
    Let TP be the true positives, TN the true negatives, FP the false positives, and FN the false negatives. The metric F1-score is defined as:
    \begin{equation} 
        F_{1-score}=\frac{2TP}{2TP+FP+FN}=2\frac{precision\cdot recall}{precision+recall}.
    \end{equation}
\end{definition}

MCC, given in Definition~\ref{eq:MCC}, computes the Pearson product moment correlation coefficient in a contingency matrix method. Its worst value is -1, and its best value is 1, 0 corresponds to a random classification. This metric is stronger against unbalanced datasets than F1-score and accuracy \cite{chicco2020advantages}. 

\begin{definition}[MCC] \label{eq:MCC}
    Let TP be the true positives, TN the true negatives, FP the false positives, and FN the false negatives. The metric MCC is defined as:
    \begin{equation}
        MCC=\frac{TP\cdot TN - FP \cdot FN}{\sqrt{(TP+FP)\cdot(TP+FN)\cdot(TN+FP)\cdot(TN+FN)}}
    \end{equation}
\end{definition}

The testing set is composed of two classes: core inliers and synthetic near-anomalies. Note that the core inliers testing dataset may also be corrupted by real unknown near-anomalies because it is based on real world data, so comparing the metrics for baselines and CANARI only makes sense on our datasets.

\subsection{Results}
\subsubsection{Observation of the growth of the eCF with n}
The goal of this subsection is to visually assess the eCF growth patterns for fail, pass, and near-anomaly samples. Pretraining and joint-training steps of CLOE are performed for $n_{CLOE}=4$. Then the eCF is computed for $n$ varying from $1$ to $5$. Figure~\ref{fig:growthDetected} (a) illustrates this experiment with PCBs that fail the ICT represented by the red graphs (only 2 red graphs because only 2 such PCBs) and PCBs that pass the ICT represented by green graphs. The testing set includes fail products in this experiment. The fail samples (in red in the image) exhibit the greatest increase, even for low values of $n$. Then, between $n=3$ and $n=4$, the graphs of some pass samples show a change in growth pattern, these products are the near-anomaly products we are looking for. The other pass samples have a polynomial growth from $n=1$ to $n=5$.

With this experiment, the change in the growth pattern with the parameter $n$ is visible. Part of the distribution is clearly different from the rest of the distribution.

\subsubsection{Evaluation of the proposed method}
In this section, and for computational reasons, $n_1$ is fixed to 1 and $n_2$ is fixed to 4. Different values for the parameter $k$ are tested, namely $k \in \{1, \sqrt{2}, 1.5, 2, 2\sqrt{2}, 3, 4, 5, 10\}$. According to the Chebyshev inequality~\cite{pukelsheim1994three}, this values correspond to an upper bound on the proportion of samples falling below the defined threshold $\tau_{na}$, as summarized in Table~\ref{tab:chebyshev}. A slightly smaller proportion falls below $\tilde{\tau}_{na}$. The value of $\tau_{synth}$ is fixed to $10^{-2}$ to create the disturbed magnitude of the near anomaly samples. 

\begin{table}[t]
    \centering
    \begin{tabular}{c|ccccccccc}
        $k$ &1&$\sqrt{2}$&$1.5$&$2$&$2\sqrt{2}$&$3$&$4$&$5$&$10$\\
        \hline
        Maximum percentage &$100\%$&$50\%$&$44.5\%$&$25\%$&$12.5\%$&$11.11\%$&$6.25\%$&$4\%$&$1\%$
    \end{tabular}
    \caption{Correspondence between the parameter $k$ and the proportion of samples outside the upper bound according to the Chebyshev inequality~\cite{pukelsheim1994three}.}
    \label{tab:chebyshev}
\end{table}

Figure~\ref{fig:growthDetected} (b) shows the growth of the eCF with $n$ varying from 1 to 5. The blue graphs correspond to the near-anomaly samples detected by CANARI while the green ones correspond to the samples detected as core inliers. We can observe that the detected near-anomaly samples correspond to samples whose growth lies between that of fail PCB samples and core inlier PCB samples. This observation confirms that CANARI detects samples with a growth between fail samples and core inlier samples.

\begin{figure}
    \begin{minipage}[c]{0.8\textwidth}
        \centering 
        \includegraphics[width=\linewidth]{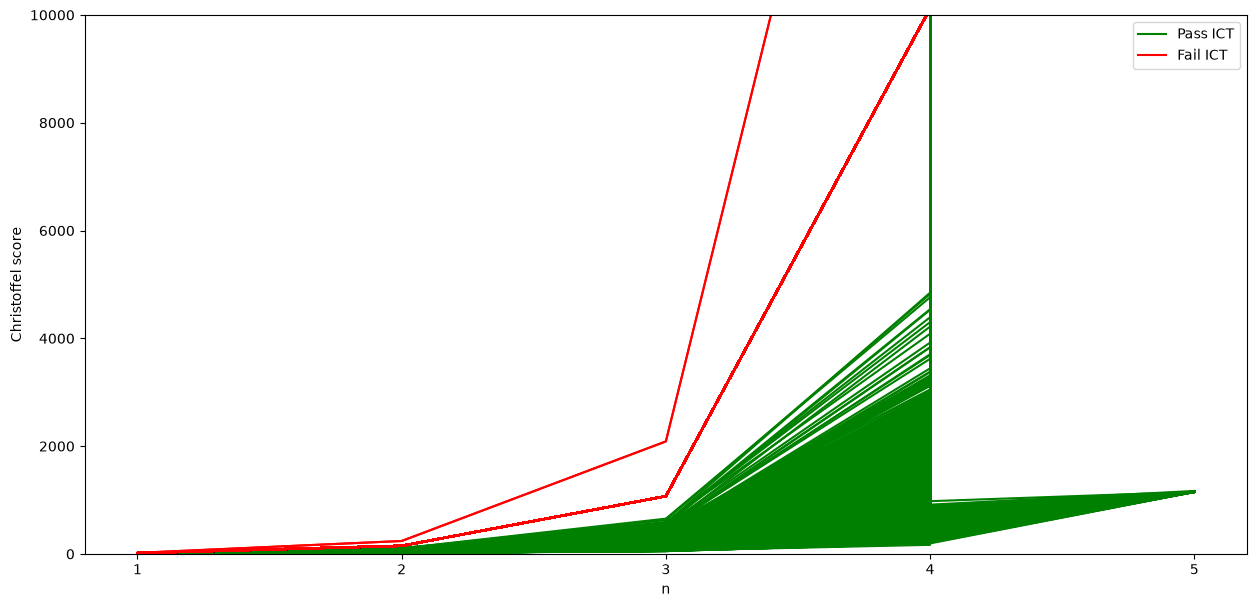} 
        \caption*{(a) Pass ICT in green, and Fail ICT in red.}
        \vspace{0.3cm}
        \includegraphics[width=\linewidth]{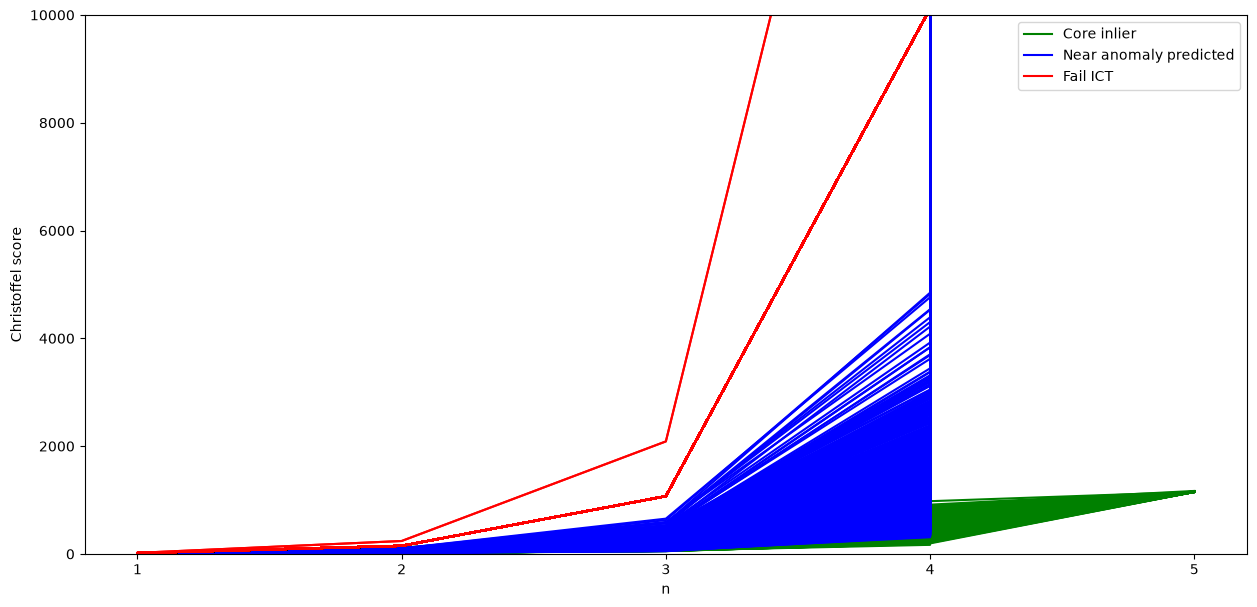}
        \caption*{(b) Predicted core inliers (green) and near anomalies (blue) and real fail samples (red).} 
    \end{minipage}
    \caption{The figure highlights the evolution of the eCF as $n$ varies from 1 to 5. (a) displays the true label, PASS or FAIL, with the samples that succeeded ICT in green and the samples that failed ICT in red. In this dataset, only two products have failed the ICT, hence only two red graphs. (b) displays the predicted label, with core inliers in green and near anomalies in blue. In red, we find the two Fail ICT samples kept to compare their graph growth. For $n=5$, core inliers have slightly different scores around 1500, hence the "strange" value concentration due to the scale. Between $n=4$ and $n=5$, the near anomaly sample scores increase so sharply that their values are out of the figure for $n=5$, producing a "strange" blue line that appears to be vertical at $n=4$ in the figure.
    Red graphs, i.e., those of Fail ICT samples, exhibit a faster growth than other graphs. The difference between the growth of the green and blue graphs, i.e., core inliers and near anomalies respectively, really occurs in the interval between $n=3$ to $n=4$. Over the interval from $n=1$ to $n=5$, we can hence distinguish a polynomial growth for the green graphs, i.e., those of core inliers, compared to an exponential growth for the red graphs, i.e., those of Fail ICT samples, the near anomaly graphs growth standing in between.}
    \label{fig:growthDetected}
\end{figure}

Then synthetic datasets were created for PCBs as described in Section~\ref{sec:dataset}. 
Five runs are performed with different random seeds. The random seed is the same for each dataset and each method to ensure reproducibility across the same run. The mean and the standard deviation across these five runs are displayed in Table~\ref{tab:experiencesf1score} for the F1-score metric and in Table~\ref{tab:experiencesMCC} for the MCC metric.

The performance obtained by CANARI on the datasets D1 to D8 is better than double thresholded CLOE with different percentage thresholds for both metrics. For datasets D1 to D8, CANARI improves the F1-score by 65.8\% and the MCC by 34.3\% comparing to the CLOE methods with percentage. Classic CANARI obtains better performances than CANARI-$\tau_{mean}$ and CANARI-$\tau_{upper}$. This shows that the proposed threshold $\tilde\tau_{na}$ can actually detect near anomaly samples.

\begin{table}[t]
    \centering \setlength{\tabcolsep}{2pt} 
    \resizebox{\textwidth}{!}{%
    \begin{tabular}{c|c|c|c|c|c|c|c|c||c} 
        Method & D1 & D2 & D3 & D4 & D5 & D6 & D7 & D8 & Mean \\
        \midrule
         CANARI &\textbf{0.76\tiny{$\pm$0.1}}&\textbf{0.49\tiny{$\pm$0.06}}&\textbf{0.64\tiny{$\pm$0.15}}&\textbf{0.45\tiny{$\pm$0.04}}&\textbf{0.41\tiny{$\pm$0.06}}&\textbf{0.50\tiny{$\pm$0.06}}&\textbf{0.50\tiny{$\pm$0.25}}&\textbf{0.72\tiny{$\pm$0.08}}&\textbf{0.56}\\
         CANARI-${\tau_{mean}}$&0.16\tiny{$\pm$2.10$^{-6}$}&0.21\tiny{$\pm$0.002}&0.16\tiny{$\pm$10$^{-6}$}&0.31\tiny{$\pm$0.01}&0.20\tiny{$\pm$0.004}&0.22\tiny{$\pm$0.003}&0.03\tiny{$\pm$0.003}&0.07\tiny{$\pm$4.10$^{-6}$}&0.17\\
         CANARI-${\tau_{upper}}$ &0.25\tiny{$\pm$0.0004}&0.21\tiny{$\pm$0.0009}&0.22\tiny{$\pm$0.001}&0.14\tiny{$\pm$0.01}&0.07\tiny{$\pm$0.006}&0.17\tiny{$\pm$0.02}&0.49\tiny{$\pm$0.06}&0.45\tiny{$\pm$0.001}&0.25\\
         CLOE$_{5\%}$ &0.16\tiny{$\pm$0.02}&0.06\tiny{$\pm$0.006}&0.04\tiny{$\pm$0.0004}&0.17\tiny{$\pm$0.01}&0.15\tiny{$\pm$0.01}&0.17\tiny{$\pm$0.01}&0.34\tiny{$\pm$0.03}&0.34\tiny{$\pm$0.04}&0.18\\
         CLOE$_{20\%}$ &0.26\tiny{$\pm$0.02}&0.09\tiny{$\pm$0.002}&0.08\tiny{$\pm$0.002}&0.28\tiny{$\pm$0.02}&0.16\tiny{$\pm$0.01}&0.24\tiny{$\pm$0.01}&0.37\tiny{$\pm$0.04}&0.39\tiny{$\pm$0.03}&0.23\\
         CLOE$_{50\%}$ &0.45\tiny{$\pm$0.01}&0.19\tiny{$\pm$0.005}&0.27\tiny{$\pm$0.02}&0.39\tiny{$\pm$0.004}&0.25\tiny{$\pm$0.003}&0.31\tiny{$\pm$0.02}&0.41\tiny{$\pm$0.05}&0.45\tiny{$\pm$0.03}&0.34\\
    \end{tabular}}
    \caption{F1-score for the 8 datasets D1 to D8 for our method and the baselines method} \label{tab:experiencesf1score}
\end{table}

\begin{table}[t] 
    \centering \setlength{\tabcolsep}{2pt}
    \resizebox{\textwidth}{!}{%
    \begin{tabular}{c|c|c|c|c|c|c|c|c||c}
        Method & D1 & D2 & D3 & D4 & D5 & D6 & D7 & D8 & Mean \\
        \midrule
    CANARI&\textbf{0.76\tiny{$\pm$0.09}}&\textbf{0.48\tiny{$\pm$0.07}}&\textbf{0.61\tiny{$\pm$0.17}}&\textbf{0.47\tiny{$\pm$0.04}}&\textbf{0.39\tiny{$\pm$0.08}}&\textbf{0.49\tiny{$\pm$0.04}}&\textbf{0.50\tiny{$\pm$0.25}}&\textbf{0.73\tiny{$\pm$0.08}}&\textbf{0.55}\\
        CANARI-${\tau_{mean}}$&0.09\tiny{$\pm$10$^{-5}$}&0.15\tiny{$\pm$0.003}&0.09\tiny{$\pm$9.10$^{-5}$}&0.28\tiny{$\pm$0.01}&0.17\tiny{$\pm$0.006}&0.19\tiny{$\pm$0.005}&0.06\tiny{$\pm$0.001}&0.07\tiny{$\pm$9.10$^{-6}$}&0.14\\
         CANARI-${\tau_{upper}}$ &0.37\tiny{$\pm$0.003}&0.32\tiny{$\pm$0.001}&0.32\tiny{$\pm$0.002}&0.21\tiny{$\pm$0.02}&0.11\tiny{$\pm$0.01}&0.23\tiny{$\pm$0.03}&0.48\tiny{$\pm$0.05}&0.53\tiny{$\pm$0.0006}&0.32\\
         CLOE$_{5\%}$ &0.27\tiny{$\pm$0.01}&0.16\tiny{$\pm$0.001}&0.14\tiny{$\pm$0.001}&0.27\tiny{$\pm$0.01}&0.22\tiny{$\pm$0.02}&0.28\tiny{$\pm$0.01}&0.40\tiny{$\pm$0.04}&0.43\tiny{$\pm$0.03}&0.27\\
         CLOE$_{20\%}$ &0.36\tiny{$\pm$0.02}&0.20\tiny{$\pm$0.003}&0.19\tiny{$\pm$0.004}&0.37\tiny{$\pm$0.01}&0.24\tiny{$\pm$0.02}&0.34\tiny{$\pm$0.01}&0.10\tiny{$\pm$0.04}&0.47\tiny{$\pm$0.02}&0.28\\
         CLOE$_{50\%}$ &0.52\tiny{$\pm$0.008}&0.30\tiny{$\pm$0.003}&0.37\tiny{$\pm$0.02}&0.45\tiny{$\pm$0.004}&0.32\tiny{$\pm$0.005}&0.39\tiny{$\pm$0.01}&0.43\tiny{$\pm$0.05}&0.52\tiny{$\pm$0.02}&0.41\\
    \end{tabular}}
    \caption{MCC for the 8 datasets D1 to D8 for our method and the baselines method} \label{tab:experiencesMCC} 
\end{table}

\subsection{Ablation and sensitivity studies}
An ablation study is performed to evaluate the robustness of CANARI. This section relies exclusively on MCC as the evaluation metric. First, Table~\ref{tab:ablation_k} reports a sensitivity analysis of parameter $k$ across datasets. The optimal values for $k$ across the different datasets are $\sqrt{2}$, $1.5$ and $2$. Accordingly, the parameter $k$ should be selected among these three values.

Second, Figure~\ref{fig:ablation_n} compares different pairs of parameters $n_1$ and $n_2$, under the constraint $n_1<n_2$. This heatmap shows the mean MCC associated with each possible combination of $n_1$ and $n_2$ across all datasets. The best combination is the one used in CANARI, namely $n_1$ = 1 and $n_2$ = 4.

Then, Table~\ref{tab:ablation_cristal} compares CLOE with the kernelized inverse Christoffel function (KIC)~\cite{askari2018kernel} and the univariate Christoffel function~\cite{grivet2026scalable} as approaches for eCF computation in high dimensional datasets before applying CANARI\footnote{The implementation used for these two methods is proposed in the framework ChRISToffel Anomaly Locator available at \url{https://github.com/fgrivet/CRISTAL}}. To compute the threshold $\tilde\tau_{na}$, the best value of $k$ is selected separately for each method and dataset. For KIC~\cite{askari2018kernel} and UCF~\cite{grivet2026scalable} and for some datasets, the minimum value or the $99\%$ percentile of the eCF $\Lambda^{\mu_N}_{n_1}(\mathbf{x})^{-1}$ can be very low, in the order of $10^{-68}$, and the ratio between the two Christoffel functions cannot be computed. These datasets are therefore not included in the present analysis. CLOE consistently achieves the best results across all datasets. Only the KIC method yields similar performance on the last dataset.

Finally, Table~\ref{tab:ablation_o} presents a comparison of several values of $\tau_{synth}$, corresponding to the disturbance magnitude used for generating the synthetic near-anomaly dataset. Except for one dataset with $\tau_{synth}=0.1$, CANARI is able to detect synthetic near-anomaly samples generated with amplitudes between $0.1$ and $10^{-4}$, which illustrates the robustness of the method that does not need to define a strict near-anomaly threshold for the synthetic near anomalies.

\begin{table}[t]
    \centering
    \begin{tabular}{c|c|c|c|c|c|c|c|c}
        k & D1 & D2 & D3 & D4 & D5 & D6 & D7 & D8 \\
        \midrule
         $1$&0.743&0.342&0.539&0.418&0.293&0.442&0.318&0.677\\
        $\sqrt{2}$ &0.750&0.446&\textbf{0.610}&0.445&0.351&\textbf{0.494}&0.415&\textbf{0.733}\\
        $1.5$ &\textbf{0.759}&0.435&0.609&0.441&0.344&0.487&0.402&0.725\\
         $2$ &0.712&\textbf{0.478}&0.606&\textbf{0.469}&\textbf{0.387}&0.493&0.472&0.687\\
         $2\sqrt{2}$ &0.640&0.461&0.565&0.462&0.326&0.450&0.496&0.602\\
         $3$ &0.632&0.451&0.551&0.443&0.325&0.448&\textbf{0.501}&0.601\\
         $4$ &0.536&0.398&0.458&0.428&0.257&0.437&0.496&0.575\\
         $5$ &0.479&0.343&0.386&0.386&0.257&0.425&0.476&0.552\\
         $10$ &0.053&0.138&0.088&0.051&0.022&0.108&0.434&0.151\\
    \end{tabular}
    \caption{MCC for the 8 datasets D1 to D8 for different values of the parameter $k$} \label{tab:ablation_k}
\end{table}

\begin{figure}[t] 
    \centering
    \includegraphics[width=0.6\linewidth]{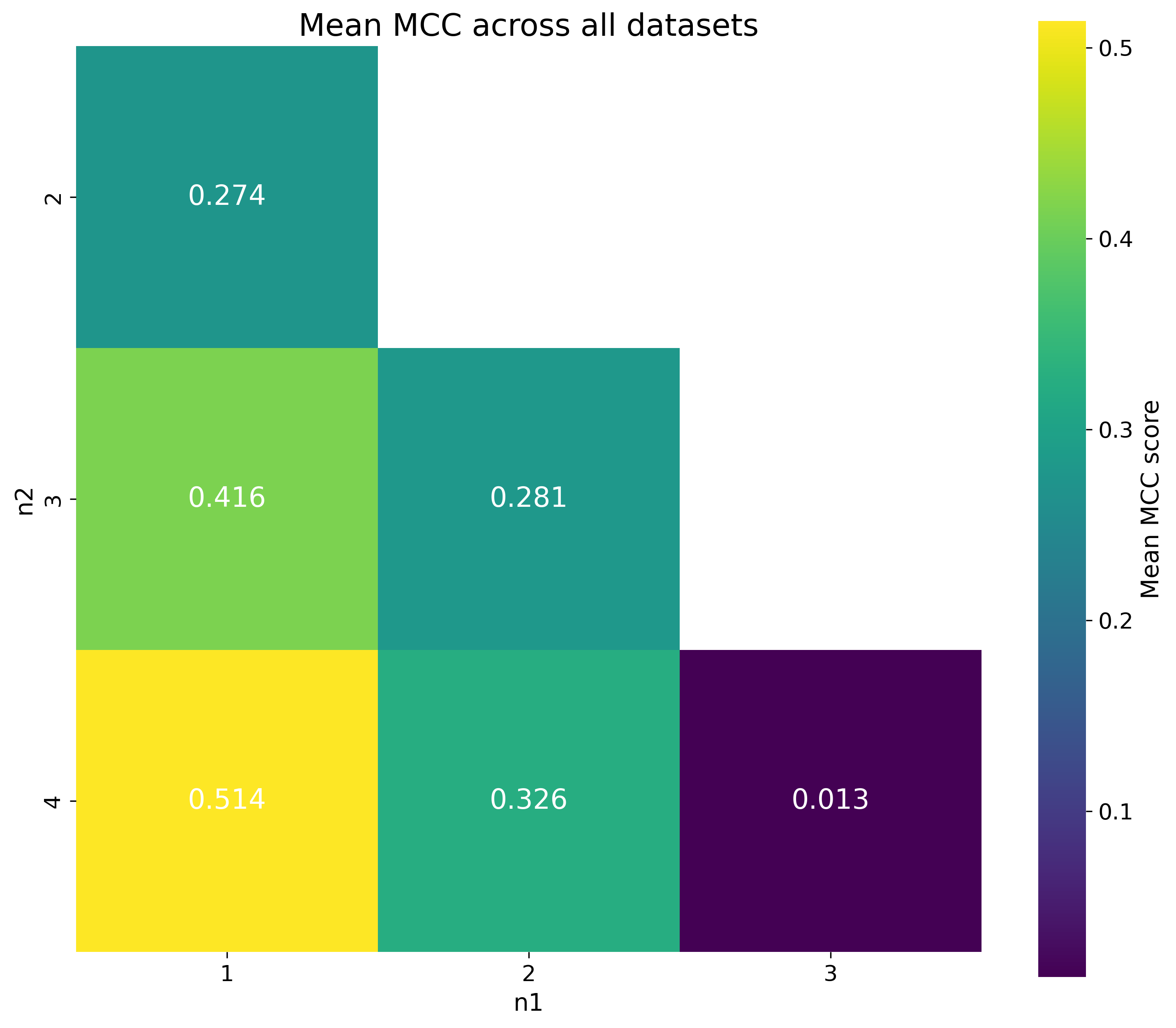}
    \caption{Heatmap of the mean of the MCC on all the datasets for different values of $n_1$ and $n_2$. The best combination is obtained for $n_1$=1 and $n_2$=4, which are the values that are used in CANARI.}
    \label{fig:ablation_n}
\end{figure}

\begin{table}[t]
    \centering
    \begin{tabular}{c|c|c|c|}
        Method to reduce dimension & D1 & D3 &D8 \\
        \midrule
        $CANARI_{CLOE}$&\textbf{0.76\tiny{$\pm$0.09}}&\textbf{0.61\tiny{$\pm$0.17}}&\textbf{0.73\tiny{$\pm$0.08}}\\
        $CANARI_{KIC}$&0.56\tiny{$\pm$0.01}&0.31\tiny{$\pm$0.09}&0.69\tiny{$\pm$0.007}\\
         $CANARI_{UCF}$ &0.31\tiny{$\pm$0.08}&0.41\tiny{$\pm$0.11}&0.29\tiny{$\pm$0.26}\\
    \end{tabular}
    \caption{MCC for the 3 datasets D1, D3, and D8 for different methods to reduce dimension before applying CANARI} \label{tab:ablation_cristal}
\end{table}

\begin{table}[t]
    \centering
    \begin{tabular}{c|c|c|c|c|c|c|c|c}
        $\tau_{synth}$ & D1 & D2 & D3 & D4 & D5 & D6 & D7 & D8 \\
        \midrule
         $0.1$&0.61 &0.33 &0.69 &0.34 &0.40 &0.45 &0 &0.57\\
         $10^{-2}$ &\textbf{0.76} &\textbf{0.48} &0.61 &\textbf{0.47} &0.39 &\textbf{0.49} &0.50&\textbf{0.73} \\
         $10^{-3}$ &0.62 &0.34 &\textbf{0.73}&0.34 &0.48 &0.47 &0.59 &0.58 \\
         $10^{-4}$ &0.62 & 0.33&\textbf{0.73} &0.36 &\textbf{0.49} &0.47 & \textbf{0.61}&0.55 \\
    \end{tabular}
    \caption{MCC for the 8 datasets D1 to D8 for different values of $\tau_{synth}$} \label{tab:ablation_o}
\end{table}

\section{Conclusion and Future Work} \label{sec:conclusion}
This paper presents CANARI and introduces the near-anomaly concept. Near-anomalies are samples that, while not yet anomalous, lie close to the boundary of the distribution and are likely
to transition into anomalies in the near future. CANARI is an unsupervised near-anomaly detection method based on the eCF that leverages theoretical properties of the CF distinguishing its growth with the parameter $n$ for data in the support and data outside the support. CANARI is constructed with two eCFs with two different values for the parameter $n$ from which the growth can be assessed. A threshold based on the Chebyshev inequality and on theoretical properties of the eCF is proposed to detect near-anomaly samples. In high-dimensional settings, CANARI can be coupled with CLOE \cite{billet2026cloe} to reduce the data dimension.
 
To evaluate this method, industrial datasets composed of ICT from a PCB manufacturing process are considered. Synthetic datasets are created to simulate near-anomaly samples based on real industrial scenarios. Experiments confirm that CANARI detects near-anomaly samples more efficiently than a classical anomaly detector with double thresholds. Moreover, CANARI with the presented threshold achieves better performance than CANARI with other thresholds.

A sensitive analysis and an ablation study were conducted to assess the robustness of CANARI. The optimal values of parameter $k$ are identifies through this study. The selected combination of $n_1$ and $n_2$ corresponds to the one that yields the best average performance across all datasets. The variation in the amplitude used to generate the synthetic near-anomalies dataset confirms the stability of CANARI. Finally, CLOE is compared with two alternative methods for computing the eCF in high-dimensional dataset. The performance obtained by the two other approaches support the choice of CLOE made in this paper. Moreover, the two alternative methods fail to compute the ratio on five datasets because of very small eCF values.

The CANARI method can be adapted to any quality process with many tests with defined limits. These quality processes are commonly used in industry, where thresholds are defined on the basis of physical models. Such physical models fail to consider the production distribution and consequently reduce the ability to detect other problems. It can be used to anticipate future failures in products that initially look normal but can have some invisible defects and cause issues at a later step.

\bibliography{ref}

\end{document}